\documentclass[conference]{IEEEtran}
\IEEEoverridecommandlockouts
\usepackage{amsmath,graphicx}
\usepackage{amsfonts}
\usepackage{array}
\usepackage{textcomp}
\usepackage{stfloats}
\usepackage{url}
\usepackage{verbatim}
\usepackage{amsthm}
\usepackage{epstopdf}
\usepackage{color}
\usepackage{amscd}
\usepackage{atbegshi}
\usepackage{dsfont}
\usepackage{amssymb}
\usepackage{float}
\usepackage{mathtools}
\usepackage{caption}
\usepackage{subcaption}
\usepackage{grffile}

\usepackage{algpseudocodex}
\algrenewcommand{\algorithmicrequire}{\textbf{Input:}}
\algrenewcommand{\algorithmicensure}{\textbf{Output:}}
\usepackage{algorithm}

\usepackage{bm}
\usepackage{caption}
\usepackage{xcolor}
\usepackage{soul}
\newcommand{\Ex}{\mathbb{E}}
\newcommand{\bmtheta}{\bm{\theta}}

\newcommand{\matW}{\mathbf{W}}

\newcommand{\vecx}{\mathbf{x}}
\newcommand{\vecy}{\mathbf{y}}

\newcommand{\vecv}{\mathbf{v}}
\newcommand{\vecw}{\mathbf{w}}
\newcommand{\vecz}{\mathbf{z}}

\newcommand{\vecs}{\mathbf{s}}
\newcommand{\vecr}{\mathbf{r}}
\newcommand{\vecg}{\mathbf{g}}
\newcommand{\tsum}{\textstyle{\sum}}

\makeatletter
\@addtoreset{theorem}{section}
\makeatother

\newcommand{\removelatexerror}{\let\@latex@error\@gobble}

\newtheorem{theorem}{Theorem}

\newtheorem{lemma}{Lemma}

\def\BibTeX{{\rm B\kern-.05em{\sc i\kern-.025em b}\kern-.08em
    T\kern-.1667em\lower.7ex\hbox{E}\kern-.125emX}}
\begin{document}

\title{Over-the-air Clustered Wireless Federated Learning \thanks{We acknowledge research grants from MeITY, Power grant from  DST SERB, and PMRF research grant from the  govt. of India.}
\vspace{-4mm}}

\author{\IEEEauthorblockN{Ayush Madhan-Sohini*}
\IEEEauthorblockA{\textit{ECE, IIITD} \\
New Delhi, India \\
ayush19156@iiitd.ac.in}
\and
\IEEEauthorblockN{Divin Dominic*}
\IEEEauthorblockA{\textit{ECE, IIITD} \\
New Delhi, India \\
divin19163@iiitd.ac.in}
\and
\IEEEauthorblockN{Nazreen Shah}
\IEEEauthorblockA{\textit{ECE, IIITD} \\
New Delhi, India \\
nazreens@iiitd.ac.in}
\and
\IEEEauthorblockN{Ranjitha Prasad
\IEEEauthorblockA{\textit{ECE, IIITD} \\
New Delhi, India \\
ranjitha@iiitd.ac.in}
}
\vspace{-5mm}
}

\maketitle

\begin{abstract}
Privacy and bandwidth  constraints have led to the use of federated learning (FL) in wireless systems, where training a machine learning (ML) model is accomplished collaboratively without sharing raw data. While using bandwidth-constrained uplink wireless channels, over-the-air (OTA) FL is preferred since the clients can transmit parameter updates simultaneously to a server. A powerful server may not be available for parameter aggregation due to increased latency and server failures. In the absence of a powerful server, decentralised strategy is employed where clients communicate with their neighbors to obtain a consensus ML model while incurring huge communication cost. In this work, we propose the OTA semi-decentralised clustered wireless FL (CWFL) and CWFL-Prox algorithms, which is communication efficient as compared to the decentralised FL strategy, while the parameter updates converge to global minima as $\mathcal{O}(1/T)$ for each cluster. Using the MNIST and CIFAR10 datasets, we demonstrate the accuracy performance of CWFL is comparable to the central-server based COTAF and proximal constraint based methods, while beating single-client based ML model by vast margins in accuracy. 
\end{abstract}

\vspace{-3mm}
\section{Introduction}
\label{sec:intro}

The proliferation of wireless devices in our daily lives has led to the rapid evolution of capabilities of wireless technologies towards greater network coverage, higher throughput, and lower latency while supporting high user densities. Naturally, an array of wireless protocols and standards cater to coexistence and interoperability, further leading to  data-driven machine learning (ML) solutions in  wireless systems.

Conventional ML approaches learn the model at a central entity (referred to as a server) irrespective of the source of data, i.e., if the source is an edge device (referred to as a client), then the data at the client is transmitted to the server for model training. Transmission of raw data is not always feasible in wireless communications due to bandwidth and privacy constraints. This naturally triggers the idea of distributed learning approaches that retain data at clients. Federated learning (FL) is one such distributed  approach that is particularly suited to tackle the aforementioned challenges in wireless communications. In conventional FL, clients train a parametric ML model, while the server periodically collects these models and aggregates the parameters to form a global model. Subsequently, these models are broadcasted back to the clients for downstream tasks. 

While employing FL in wireless communications, the base station plays the role of the server \cite{amiri2021convergence}, which implies that the parameters from the clients are sent over a resource-constrained uplink channel. Furthermore, the multi-user nature of FL also necessitates the use of orthogonal time and frequency resources. For example, in frequency division multiplexing (FDM), each user is  assigned a dedicated bandwidth, leading to a bandwidth deficit and an increase in energy consumption if there are numerous participating clients. Hence, repeated communication of the local models to the server entails a severe load on this uplink channel. Several strategies, such as sparsification and quantization, are adopted for efficiently communicating local models \cite{reisizadeh2020fedpaq}. 

A popular paradigm for efficient communication over a common uplink multiple access channel (MAC) is over-the-air FL (OTA-FL) strategy. OTA-FL allows clients to simultaneously transmit updates using analog signaling over the uplink channel in a non-orthogonal manner, hence optimizing available temporal and spectral resources \cite{COTAF}. Several variations of the OTA-FL algorithm have been proposed \cite{amiri2021convergence,yang2020federated}. In \cite{COTAF}, the authors propose the COTAF algorithm, which uses a novel pre-coding technique to facilitate high throughput OTA-FL over wireless channels. This is the first-of-its-kind technique that achieves $\mathcal{O}(1/T)$ convergence in parameter updates in the presence of noise. 

Since its inception, centralised FL where the server creates the global model has been popular in the literature owing to superior performance and convergence guarantees. However, such a centralised approach is vulnerable to increased latency due to bottlenecks and server failures. For example, in D2D communications, a base station-centric architecture is not feasible due to connectivity and computational constraints. In such scenarios, a decentralised wireless FL architecture is lucrative \cite{DecentraFLSimeone} since the clients can communicate with each other directly. In decentralised learning, each client  performs local SGD steps using a mini-batch of data drawn from its local dataset and communicates the updated model parameters to its neighboring clients. Each client incorporates the information received from the neighbors with its local information using a mixing matrix to achieve consensus. However, using the mixing matrix forces decentralised techniques to rely on communication over orthogonal uplink channels \cite{DecentraFLSimeone}, i.e., without a server, communication efficiency is poor and scales as $\mathcal{O}(K^2)$. Hence, it is of utmost importance to devise decentralised FL schemes that are communication efficient while employing wireless uplink channels. 

 \noindent \textbf{Contributions:} We propose a novel communication-efficient semi-decentralised OTA-FL strategy, which we refer to as  clustered wireless federated learning (CWFL). CWFL involves decentralised data-agnostic clustering of clients using the following steps: (a) Representative members (mid-tier computational devices) of each cluster called cluster-heads build a cluster-level ML model using OTA communication, (b) Cluster-heads exchange model updates among each other in a decentralised fashion to arrive at a consensus ML model. Briefly, the contributions are as follows: 
\begin{itemize}
    \item CWFL has an improved communication complexity of $\mathcal{O}(C^2)$, where $C$ is the number of cluster-heads. 
    \item CWFL with a proximal constraint (CWFL-Prox) has improved robustness to statistical heterogeneity in the presence of noise. 
    \item In the presence of statistically heterogeneous data, CWFL achieves $\mathcal{O}(1/T)$ convergence, similar to counterparts \cite{COTAF} where the server is present.
\end{itemize} 
There are several real-world use-cases where the decentralised set-up is crucial. Examples include D2D wireless networks where the agents are connected in a decentralized topology, D2D relay clustering system \cite{D2Drelayclustering}, wireless fog networks where the distributed computing paradigm encapsulates communication among edge devices such as local area servers, UAVs and cloud servers, and the wireless sensor network-inspired-IoT networks \cite{CWFL_WSN}. Using the MNIST and CIFAR10 datasets, we demonstrate that CWFL and CWFL-Prox frameworks perform similar (with respect to model accuracy) to the centralized FL technique \cite{COTAF}. We also argue that CWFL-type architecture is more preferred to COTAF (in terms of accuracy) under certain conditions. By design, our scheme has a lower communication complexity than the decentralised framework \cite{DecentraFLSimeone}. A non-trivial aspect of novelty is that of convergence, which is achieved using the proposed novel power control mechanism.
\vspace{-2mm}
\section{System Model}
\vspace{-1mm}
We consider a wireless multi-user system with $K$ clients. Each participating client has access to a data set $\mathcal{D}_k$, which consists of $N_k$ instances, i.e., the total number of instances is given by $N  = \sum_{k = 1}^K N_k$.  In supervised learning, the dataset at the $k$-th client, $\mathcal{D}_k$, consists of data samples given as a set of input-output pairs $\{\vecz_i, y_i\} \in \mathcal{D}_k$ for $i = [N_k]$, where $\vecz_i \in \mathbb{R}^m$, and $y_i \in \mathbb{R}$ is the label for the sample $\vecz_i$. This data may be generated at the clients via onboard sensors or interactions with mobile apps. A typical learning problem is to find the model parameter $\bmtheta_k \in \mathbb{R}^d$ by optimizing the empirical loss function on, $\mathcal{D}_k$ given by
\begin{equation}
    f_k(\bmtheta_k) \triangleq \tfrac{1}{N_k}\tsum_{i \in \mathcal{D}_k} l(\vecz_i;\bmtheta_k),
\end{equation}
where $l(\vecz_i;\bmtheta_k)$ is the loss per instance. In the conventional server-based FL framework, the global ML model is assumed to be parameterised by $\bmtheta_k \in \mathbb{R}^d$ itself, which leads to a global objective function given as
\begin{align}
    \min_{\bmtheta_k} F(\bmtheta_k) \triangleq \tfrac{1}{K}\tsum_{k = 1}^K f_k(\bmtheta_k).
    \label{eq:FL_basic}
\end{align}

\subsection{Wireless Federated Learning}

In wireless communications, a base-station based server communicates with $K$ \emph{wireless} clients.  The parameter updates from the clients to the server, and global synchronization from server to client happens over the resource constrained uplink and downlink channels, respectively. Typically, downlink transmission is assumed to be error-free \cite{COTAF} owing to availability of sophisticated error-control methods. Further, it is assumed that communication takes place over $T$ rounds, where the number of channel uses in each round depends upon the FL strategy. For $t \in [T]$ and $k \in [K]$, the signal received at the server via the uplink AWGN channel is given by
\begin{align}
    \vecy^t = \tsum_{k = 1}^K \vecx^t_k + \vecw^t, ~\textnormal{where}~~ 0<\Ex[\Vert{\vecx_k^t}\Vert^2] \leq P.
\label{eq:OTARx}
\end{align}
The additive noise at the server is modeled as $\vecw^t \sim \mathcal{N}(0,\sigma^2\mathbf{I}_d)$, where $\sigma^2$ is the noise variance. We assume that the input is power constrained
and $P$ represents the available transmission power, i.e., $\vecx_k^t$ is given by
\begin{align}
    \vecx_k^t = \sqrt{p^t}(\bmtheta_k^t-\bmtheta^{g}), ~\text{where}~ p^t \triangleq \tfrac{P}{\max_k\Ex[\Vert{\bmtheta^{t}_k-\bmtheta^{g}}\Vert^2]},
    \label{eq:precoding_x}
\end{align}
where $\bmtheta_k^t$ is the parameter update at the $k$-th client for the $t$-th communication round, and $\bmtheta^{g}$ is the global parameter update upto $t$. Here, $p^t$ is a pre-coding factor that scales the model parameters as $t$ progresses such that the power constraint stated in \eqref{eq:OTARx} is satisfied in an expected sense. 

\subsection{Over-the-Air Noisy Federated Averaging}

Federated averaging is one of the most widely adopted  algorithms \cite{FedAVG}, where the global  model is learnt at the server by aggregating the client parameters. First, the server shares its current model given by $\bmtheta^t$ with the $K$ clients. Next, each client trains on the global model ($\bmtheta^t$) to obtain a local parameter update based on one or several mini-batches of data, collectively represented as  $\omega^t_k \in \mathcal{D}_k$. The local SGD update is given by $\bmtheta^{t+1}_{k} = \bmtheta^t_{k} - \eta^t \nabla f_k(\bmtheta_k^t)$. Here, $\eta$ is the learning rate and $f_k(\bmtheta_k^t)$ is the loss evaluated at the $k$-th client.  We assume that all the clients train over $E$ epochs in one communication round $t$. The $K$ clients convey their parameter updates to the server during pre-designated aggregation time-slots $t \in \mathcal{H}$ using incremental updates, as given in \eqref{eq:OTARx}. These updates are aggregated at the server to obtain the global update, and this global update is shared with the clients. The above steps conclude one communication round, and the steps are repeated for many communication rounds until the model converges.

\noindent In the context of OTA, for $t \in \mathcal{H}$, the received signal is given by \eqref{eq:OTARx}, and the transmitted signal is as given in \eqref{eq:precoding_x} for $t \in [T]$. The decoding rule used at the server is given by
\begin{align}
    \bmtheta^t = \tfrac{\vecy^t}{K\sqrt{p^t}} + \bmtheta^{t-E} = \tfrac{1}{K}\tsum_{k = 1}^K \bmtheta^t_k + \tilde{\vecw}^t,
    \label{eq:decodeCOTAF}
\end{align}
where $\tilde{\vecw}^t \sim \mathcal{N}(0,\frac{\sigma_w^2}{K^2p^t}\mathbf{I}_d)$. Note that $\bmtheta^{t-E}$ is the previous global update referred to as $\bmtheta^g$ in \eqref{eq:precoding_x}. From~\eqref{eq:decodeCOTAF}, it is evident that after decoding, OTA leads to a noisy version of the federated averaging-based update at the server. 

In the presence of statistical heterogeneity at the clients, model learnt at the server tends to vary drastically in each round. In order to render FL robust to heterogeneity, Fedprox \cite{FedProx} introduces the proximal term as a constraint on the local objective, $f_k(\bmtheta)$,i.e., distributed optimization in \eqref{eq:FL_basic}, $f_k(\bmtheta)$ is replaced by a  proximal term $f^p_k(\bmtheta) \triangleq f_k(\bmtheta) + \frac{\lambda_p}{2}\Vert\bmtheta - \bmtheta^g\Vert^2$. The constraint ensures that the local update lies close (in a sphere) to the previous global update $ \bmtheta^g$. 
 
\begin{figure}
\centering
    \includegraphics[width=0.5\textwidth]{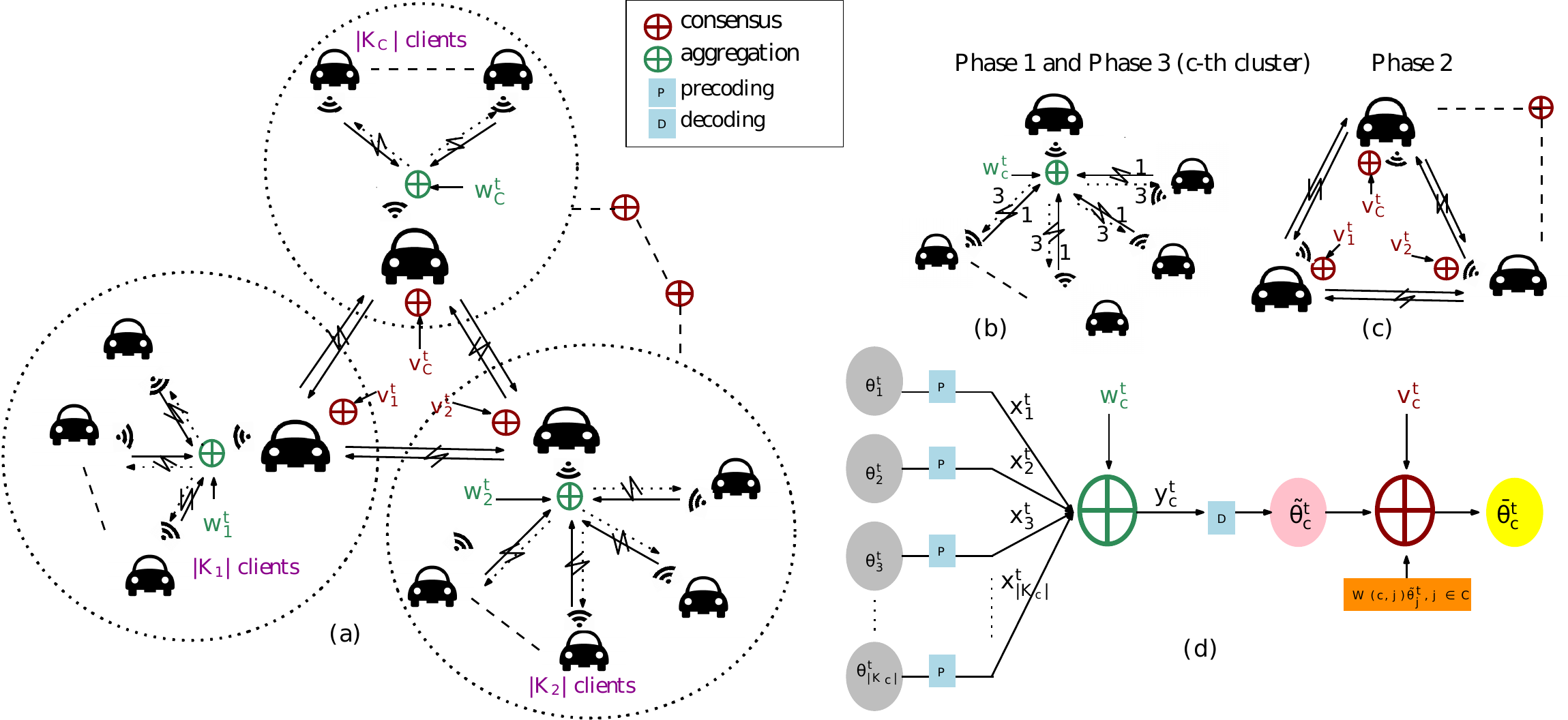}
    \caption{CWFL strategy is depicted in (a), individual phases in (b) and (c), and signal flow diagram in (d).}
    \label{fig:CWFL_phases}
\end{figure}

\section{Clustered Wireless Federated Learning}
\label{sec:CWFL}

We present the novel CWFL strategy to address the issue of communication complexity in the decentralised setting. The fundamental idea of CWFL is to cluster the clients to \emph{distribute} the tasks of a server to a \emph{subset} of clients, which we refer to as cluster-heads. Cluster-heads are the clients that are mid-tier computational devices. The clustering process is initiated by choosing $C$ cluster-heads, and clients are randomly assigned to each cluster-head such that clusters consist of non-overlapping set of clients. The clients of the $c$-th cluster are indexed by $k_c \in [K_c]$, such that $\sum_{c = 1}^C K_c = K$. CWFL is implemented in a hierarchical fashion, where the first phase of a given communication round is the uplink transmission, i.e., from the clients in each cluster to their respective cluster-heads. This phase is OTA; hence, each cluster-head takes one channel use for communication, resulting in a total of $C$ channel uses. In the second phase, the cluster-heads communicate among each other in a decentralised manner. CWFL requires $C(C-1)$ channel uses to obtain the consensus update in each communication round $t$ since every cluster-head needs one channel-use each for communication with $C-1$ possible neighbors, in the worst case. Hence, the total number of channel uses per communication round scales as $\mathcal{O}(C^2)$ unlike the decentralised strategy in \cite{DecentraFLSimeone} where the communication complexity scales as $\mathcal{O}(K^2)$.

In phase $1$ the input signal at each client is given by \eqref{eq:precoding_x}, where $\bmtheta^g$ is the cluster-level global parameter update. The corresponding channel output at the $c$-th cluster-head is given by $\vecy^t_c \in \mathbb{R}^d$ for $t \in \mathcal{H}$ is given by
\begin{align}
    \vecy^t_c = \sum_{k =1}^{K_c} \vecx_k^t + \vecw^t_c = \sqrt{p^t}\sum_{k = 1}^{K_c}(\bmtheta_k^t-\bmtheta_c^{t-E}) + \vecw^t_c  ,
    \label{eq:clushead_rx}
\end{align}
where $\vecw^{t}_c \sim \mathcal{N}(0,\sigma_c^2 \mathbf{I}_d)$ is the additive noise at the cluster head $c$. As compared to \eqref{eq:decodeCOTAF}, the above expression performs OTA parameter update at the cluster level. Given $\tilde{\vecw}_c^t \sim \mathcal{N}(0,\tfrac{\sigma_c^2}{K_c^2p^t}\mathbf{I}_d)$, the resulting parameter update at the $c$-th cluster-head is given by
\begin{align}
    \tilde{\bmtheta}^{t}_{c} = \tfrac{1}{K_c\sqrt{p^t}} \vecy^{t}_{c} + \bmtheta_c^{t-E} = \tfrac{1}{K_c}\tsum_{k =1}^{K_c}\bmtheta_k^t + \tilde{\vecw}_c^t.                      
    \label{eq:recursiveerrorform}
\end{align}
Since direct communication with the server is hindered, a decentralised learning architecture is adopted. We assume a symmetric doubly stochastic mixing matrix $\matW \in \mathbb{R}^{C \times C}$ where $W(c,j) = W(j,c)$ and $W(c,c) = 0$ for $ 1 \leq c,j \leq C$. The entries of $\matW$ encode the extent to which cluster-head $j$ can affect cluster-head $c$, while $W(c,j) = 0$ implies that cluster-heads $c$ and $j$ are disconnected. The $c$-th cluster-head transmits $\vecs_c^t \in \mathbb{R}^d$ given by 
\begin{align}
    \vecs_c^t = \sqrt{q^t}\tilde{\bmtheta}_c^t, ~~\text{where}~~{q^t}\triangleq\tfrac{P_2}{\max_c\Ex[\Vert{\tilde{\bmtheta}_c^t}\Vert^2]}.
    \label{eq:dec_precode}
\end{align}
Here, $q^t$ is a pre-coding factor that scales the model parameters so  that the power constraint given by $0\leq\Ex[\Vert{\vecs_c^t}\Vert^2]\leq{P}_2$ is satisfied for every cluster-head. The power constraint is applied on the expected  value of the transmit signal, as given in \eqref{eq:dec_precode}. In the decentralised setting, we set the parameter update shared by the neighbors of the $c$-th cluster-head, $\vecr_c^t \in \mathbb{R}^d$ as
\begin{align}
    \vecr_c^t = \tsum_{j=1}^{C} W(c,j)\vecs_j^t + \vecv_c^t, 
    \label{eq:dec_power}
\end{align}
for all $c,j \in [C]$ and $W(c,c) = 0$. Note that computing \eqref{eq:dec_power} at the $c$-th cluster-head requires $C-1$ channel uses. The received signal \eqref{eq:dec_power} at every cluster-head is corrupted by noise, which is cumulatively represented as $\vecv_c^t\sim\mathcal{N}(0,\kappa_c^2\mathbf{I}_d)$ where $\kappa_c^2$ is the variance of the additive noise as derived in Lemma~1. 

\begin{algorithm}
\caption{Clustered Wireless FL}\label{alg:cwfl}
\label{alg}
\begin{algorithmic}[1]
\Require Dataset $\mathcal{D}_k$ at the $k$-th client
\Ensure Consensus parameter $\bar{\bmtheta}_c^T$ for all $c$.
\For {$t \in [T]$}
  \If {$t \in \mathcal{H}$}
  \State  Obtain $\tilde{\bmtheta}_c^{t}$ and $\bar{\bmtheta}_{c}^{t}$ (\eqref{eq:recursiveerrorform} and \eqref{eq:dec_final}), for all $c$
  \State Cluster-head broadcasts $\bar{\bmtheta}_c^{t}$ \text{to} $\forall k_c \in [K_c], \text{for all}~c $
  \EndIf
  \If {$t \notin \mathcal{H}$}
    \State Perform local SGD on $\bar{\bmtheta}_{c}^{t}~\text{for all}~ c$
    \EndIf
  \EndFor
\Return Consensus parameter $\bar{\bmtheta}_c^T$ for all $c$.
\end{algorithmic}
\end{algorithm}
\vspace{-5mm}

\begin{lemma}
\label{lem:dec_noise}
The effective distribution of $\vecv^t_c$ is given as $\vecv^t_c \sim \mathcal{N}(0,\kappa_c^2)$, where $\kappa_c^2 = \sum_{j=1}^{C}W(c,j)\sigma^2_j\mathbf{I}_d$.
\end{lemma}

\vspace{-1mm}
\noindent In order to recover the  consensus parameter update at the $c$-th cluster head given by $\bar{\bmtheta}_c^t$, the decoding rule given by:
\begin{align}
    \bar{\bmtheta}_c^t = \tilde{\bmtheta}^t_c+\tfrac{\vecr_c^t}{\sqrt{q^t}}=\tilde{\bmtheta}^t_c+\tsum_{j=1}^{C}W(c,j)\tilde{\bmtheta}_j^t+{\tilde{\vecv}^t_c},
    \label{eq:dec_final}
\end{align}
where $\tilde{\vecv}_c^t\sim\mathcal{N}(0,\tilde{\kappa}^2_c\mathbf{I}_d)$ and $\tilde{\kappa}_c^2=\frac{1}{q^t}\sum_{j=1}^{C}W(c,j)\sigma^2_j$.
In the third phase, the cluster heads concurrently transmit the cluster-level updates to the clients in their respective clusters. The update process is summarized in Fig.~\ref{fig:CWFL_phases} and in Algorithm~1. 

\section{Convergence of CWFL }
\label{sec:majhead}

We demonstrate the convergence of the CWFL parameters using standard assumptions of $L$-Lipschitz smoothness, $\mu$-strong convexity of $f_k(\cdot)$, and  $G$-boundedness and $\alpha_k$-bounded variance of stochastic gradients \cite{COTAF}.

\begin{theorem}
Under the standard assumptions and given constants $L,\mu,\alpha_k,G$, choosing $\gamma =  \max(E,\tfrac{12L}{\mu})$ and choosing the learning rate $\eta^t = \tfrac{2}{\mu\left(\gamma+t\right)}$ for $t \in [T]$, each cluster in CWFL satisfies $\mathcal{O}(1/T)$ convergence rate since 
\begin{align}
   \Ex \lVert{\tilde{\bmtheta}_c^{T}- \bmtheta^*}\rVert^2 \leq\frac{2\max\left(4Q_1,\mu^2\gamma\lVert{\tilde{\bmtheta}_c^{0}- \bmtheta^*}\rVert^2\right)}{\mu^2(T+\gamma-1)},
   \label{eq:thm1}
\end{align}
    where $Q_1=3C\sum_{j=1}^C(W(c,j))^2P_2\mbox{A}+8E^2G^2 + 6L\Gamma +\frac{1}{K_c^2}\sum_{k=1}^{K_c}\alpha_k^2 + \frac{4d\sigma^2_cE^2G^2}{P_1K_c^2}+d\sum_{j=1}^{C}W(c,j)\sigma^2_j\mbox{A}$\\
and $\mbox{A} = \tfrac{8E^2G^2}{P_1P_2}(CP_2\sum_{j=1}^{C}(W(c,j))^2$ \\ $+d\max_c(\sum_{j=1}^{C}W(c,j)\sigma^2_j) + P_1+{(2{K}_c)}^{-1}\mathds{1}_{t\in\mathcal{H}})$.
\end{theorem}
\begin{proof}
    The virtual sequence
    from \eqref{eq:recursiveerrorform}  for $t \in [T]$ is given as
\begin{align}
    \tilde{\bmtheta}^t_c  =  \tfrac{1}{K_c}\tsum_{k=1}^{K_c}\bmtheta^t_k + \tilde{\vecw}_c^t\mathds{1}_{t\in\mathcal{H}},
\end{align}
where $\mathds{1}_{(\cdot)}$ is the indicator function. We define 
\begin{align*}
  \vecg_c^t &\triangleq \tfrac{1}{K_c}\tsum_{k = 1}^{K_c}\nabla{f(\bmtheta^t_k,\omega_k)},\;\bar{\vecg}_c^t \triangleq \tfrac{1}{K_c}\tsum_{k=1}^{K_c}\nabla{f(\bmtheta^t_k)}.
\end{align*}
Let $\hat{\vecw}_c^t\triangleq \tilde{\vecw}_c^t\mathds{1}_{t\in\mathcal{H}}$, so that $\Ex\left[\Vert\hat{\vecw}^t_c\Vert^2\right]= \Ex\left[\Vert\tilde{\vecw}_c^t\mathds{1}_{t\in\mathcal{H}}\Vert^2\right]=\frac{d\sigma^2_c}{K_c^2p^t}\mathds{1}_{t\in\mathcal{H}}$. Since SGD is used at each client, we have
\begin{align}
\tilde{\bmtheta}^{t+1}_c=\bar{\bmtheta}^t_c-\eta^t\vecg_c^t+\hat{\vecw}^t_c.
\label{eq:vir_seq}
\end{align}
The above equation differentiates our proof from previous works such as \cite{COTAF,OnConvNoniid}. Since previous works are server-based, $\tilde{\bmtheta}^{t+1}_c$ would depend upon the previous update $\tilde{\bmtheta}^{t}_c$. However, since CWFL assigns a consensus update as the global update, $\tilde{\bmtheta}^{t}_c$ is replaced by $\bar{\bmtheta}^{t}_c$ in the virtual sequence. Note that $\bar{\bmtheta}^t_c =\bmtheta^t_k$ $\forall k\in[K_c]$, when $t\in\mathcal{H}$. 

Using \eqref{eq:dec_final} and \eqref{eq:vir_seq} we have the following:
\begin{align}
   &\Vert \tilde{\bmtheta}_c^{t+1} - \bmtheta^* \Vert^2 = \Vert\bar{\bmtheta}^t_c-\eta^t{\vecg_c^t}-\bmtheta^*+{\hat\vecw}^t_c\Vert^2 = \nonumber\\
  &\Vert\tilde{\bmtheta}^t_c + \sum_{j=1}^CW(c,j)\tilde{\bmtheta}^t_c+\tilde{\vecv}^t_c-\eta^t{\vecg_c^t}-\bmtheta^*+\eta^t\bar{\vecg}_c^t-\eta^t\bar{\vecg}_c^t+{\hat\vecw}^t_c\Vert^2 \nonumber\\
   &= \underbrace{\Vert\tilde{\bmtheta}^t_c+\sum_{j=1}^CW(c,j)\tilde{\bmtheta}^t_j -\eta^t\bar{\vecg}_c^t-\bmtheta^*\Vert^2}_\text{T1}\nonumber\\
   &+\underbrace{(\eta^t)^2\Vert\bar{\vecg}_c^t-\vecg_c^t+\frac{{\hat\vecw}^t_c}{\eta^t}+\frac{{\tilde{\vecv}}^t_c}{\eta^t}\Vert^2}_\text{T2}\nonumber\\&+\underbrace{2\eta^t\langle\tilde{\bmtheta}^t_c+\sum_{j=1}^CW(c,j)\tilde{\bmtheta}^t_j-\bmtheta^*-\eta^t\bar{\vecg}_c^t,\bar{\vecg}_c^t-\vecg_c^t+\frac{{\hat\vecw}^t_c}{\eta^t}+\frac{{\tilde{\vecv}}^t_c}{\eta^t}\rangle}_\text{Expectation of this term goes to zero},
   \label{eq:start1}
   \end{align}
 where we add and subtract $\eta^t\bar{\vecg}^t_c$ to obtain the second step. 
   Following the proof in lemma A.3 of \cite{COTAF}, where N is replaced by the number of clients in the $c$-th cluster, $K_c$, when $\eta^t\leq 2\eta^{t+E}$ for all $t\geq0$ \cite{OnConvNoniid,stich2019local} and when assumption~3 holds, we can derive the following bound:
\begin{align}
    \tfrac{1}{K_c}\tsum_{k=1}^{K_c}\Ex\left[\Vert\tilde{\bmtheta}^t_c-\bmtheta^t_k\Vert^2\right]\leq 4E^2(\eta^{t})^2G^2.
    \label{lem:bound_divergence}
\end{align}
 Taking the expectation of \eqref{eq:start1} and using  Lemmas~\ref{lem:term3_divergence}, \ref{lem:bound_noisy_variance}, \ref{lem:div_non-noisy_consensus} and \eqref{lem:bound_divergence}, we obtain the following:
\begin{align*}
&\Ex\left[\Vert {\tilde{\bmtheta}}_c^{t+1} - \bmtheta^* \Vert^2\right] \leq (2-\mu\eta^t)\Ex\left[\Vert{\tilde{\bmtheta}}^t_c -\bmtheta^*\Vert^2\right]\nonumber\\
&+3(\eta^t)^2C\sum_{j=1}^C(W(c,j))^2P_2\mbox{A}\nonumber\\ &+8(\eta^t)^2E^2G^2 + 6L(\eta^t)^2\Gamma - \frac{5\eta^t}{3}\Ex\left[(F(\tilde{\bmtheta}^t_c)-F^*)\right]\nonumber\\ &+ \frac{(\eta^t)^2}{K_c^2}\left[\sum_{k = 1}^{K_c}\alpha_k^2 + \text{D}\mathds{1}_{t\in\mathcal{H}}+dK_c^2\sum_{j=1}^{C}W(c,j)\sigma^2_j\mbox{A}\right],
     \label{cotaf}
\end{align*}
where $\mbox{D} = {P_1}^{-1}4d\sigma^2_cE^2G^2$ and $\mbox{A}$ is as given in \eqref{eq:thm1}. We define $\delta^t \triangleq \Ex{\Vert\tilde{\bmtheta}^t_c - \bmtheta^*\Vert^2}$, which gives us a recursive relation similar to \cite{COTAF}. Further, since $-\eta^t\Ex\left[F(\tilde{\bmtheta}^t_c)-F^*\right]\leq0$ and as $\text{D}\mathds{1}_{t\in\mathcal{H}}\leq{\text{D}}$ for $\text{D}\geq0$. We set the step size $\eta^t=\frac{\rho}{(t+\gamma)}$ for some $\rho>\frac{1}{\mu}$ and $\gamma\geq{\max(6L\rho,E)}$, for which $\eta^t\leq\frac{1}{6L}$ and $\eta^t\leq2\eta^{t+E}$. Hence, we have
\begin{align*}
    \delta^{t+1} &\leq (2-\mu\eta^t)\delta^t+(\eta^t)^2Q_1= (1-\mu\eta^t)\delta^t+(\eta^t)^2Q_1 + \delta^t,
\end{align*}
where $Q_1$ is as given in \eqref{eq:thm1}. If $\nu\geq \frac{\rho^2Q}{\rho\mu-1}$, $\nu\geq{\gamma}{\delta}^0$ and $\delta^t\leq\frac{\nu}{t+\gamma}$, then $\delta^{t+1}\leq\frac{\nu}{t+1+\gamma}$ holds \cite{COTAF}. Further note that $\delta^t$ is always positive. It is also true that, $\delta^{t+1} \leq \frac{\nu}{t+1+\gamma} + \frac{\nu}{t+\gamma}$. This holds for $\nu = \max\left(\frac{\rho^2Q_1}{\rho\mu-1},\gamma\delta^0\right), \gamma \geq \max\left(E,6\rho{L}\right)$ and $\rho>0$. By setting, $\rho=\frac{2}{\mu}$ we have $\gamma = \max\left(E,\frac{12L}{\mu}\right),\nu =\max\left(\frac{4Q_1}{\mu^2},\gamma\delta^0\right)$ and hence,
\begin{align}
    \Ex\left[\lVert\tilde{\bmtheta}^{t+1}_c-\bmtheta^*\rVert^2\right]&\leq \frac{\max\left(4Q_1,\mu^2\gamma\delta^0\right)}{\mu^2\left(t+1+\gamma\right)}+\frac{\max\left(4Q_1,\mu^2\gamma\delta^0\right)}{\mu^2\left(t+\gamma\right)} \nonumber\\
    &\leq \frac{2\max\left(4Q_1,\mu^2\gamma\delta^0\right)}{\mu^2(t+\gamma)}.
\end{align}
Substituting $T = t+1$ in the above completes the proof. 
\end{proof}
\vspace{-2mm}
As seen above, a major challenge in implementing SGD as an OTA computation in the decentralised setting is the presence of the additive channel noise within a cluster and among the cluster-heads. We show that the effect of noise can be gradually eliminated using  additional pre-coding and scaling steps, using factors such as $p^t$ and $q^t$. The control achieved over the noise allows us to achieve a convergence rate of $\mathcal{O}(1/T)$ similar to centralised counterparts \cite{COTAF}. We state the lemmas used in this work. We omit the proofs due to lack of space. 
\begin{lemma}
 Assuming $\eta_t \leq \frac{1}{6L}$, term T1 is given by 
 \begin{align}
     &\Vert\tilde{\bmtheta}^t_c+\sum_{j=1}^CW(c,j)\tilde{\bmtheta}^t_j -\eta^t\bar{\vecg}_c^t-\bmtheta^*\Vert^2  \leq (2-\mu\eta^t)\Vert\tilde{\bmtheta}^t_c-\bmtheta^*\Vert^2\nonumber\\
     &~~~+3\lVert{\sum_{j=1}^CW(c,j)\tilde{\bmtheta}^t_j}\rVert^2+ \tfrac{2}{K_c}\sum_{k = 1}^{K_c}\Vert(\tilde{\bmtheta}^t_c-\bmtheta^t_k)\Vert^2\nonumber\\
     &+ 6L(\eta^t)^2\Gamma - \tfrac{5\eta^t}{3}(F(\tilde{\bmtheta}^t_c)-F^*),
 \end{align}
  \label{lem:term3_divergence}
 \end{lemma}
 
\begin{lemma}
   If $p^t \leq q^{t-E}$ and $\eta^t\leq2\eta^{t+E}$, we have:
   \begin{align}
    &\frac{1}{q^t}\leq\tfrac{8E^2(\eta^{t})^2G^2}{P_1}\left[C\sum_{j=1}^{C}({W(c,j)})^2+\tfrac{d}{P_2}\max_c(\sum_{j=1}^{C}W(c,j)\sigma^2_j)\right]\nonumber\\
    &+ \tfrac{8E^2(\eta^{t})^2G^2}{P_2}+\tfrac{4E^2(\eta^t)^2G^2}{K_cP_1P_2}\mathds{1}_{t\in\mathcal{H}} 
    \label{eq:dec_bound}
   \end{align}
   \label{lem:dec_power_control}
\end{lemma}

\begin{lemma}
    (Bounding the noisy variance) : If $\eta^t\leq2\eta^{t+E}$, $p^t \leq q^{t-E}$, the term T2 is given as:
\begin{align}
&\Ex\left[\Vert\bar{\vecg}_c^t-\vecg_c^t+\tfrac{\hat\vecw^t_c}{\eta_t}+\tfrac{\tilde{\vecv}^t_c}{\eta^t}\Vert^2\right]\leq\tfrac{1}{K_c^2}\sum_{k = 1}^{K_c}\alpha_k^2 \nonumber\\
&+ \tfrac{4d\sigma^2_cE^2G^2}{P_1K_c^2}\mathds{1}_{t\in\mathcal{H}}+d\sum_{j=1}^{C}W(c,j)\sigma^2_j\mbox{A}
\end{align}
where,\\ {\small$A = \frac{8E^2G^2}{P_1}\left(C\sum_{j=1}^{C}({W(c,j)})^2\\+\frac{d}
{P_2}\max_c(\sum_{j=1}^{C}W(c,j)\sigma^2_c)\right)$\\
{$~~~~~~~~~~~~~~~~~~~~~~~~~~~~~~~~~~~~~~~~~~+ \frac{8E^2G^2}{P_2}+\frac{4E^2G^2}{K_cP_1P_2}\mathds{1}_{t\in\mathcal{H}}$.}}
\label{lem:bound_noisy_variance}
 \end{lemma}

\begin{lemma}
    (Bounding the divergence in non-noisy consensus update) : If $\eta^t\leq2\eta^{t+E}$, and using Lemma~\ref{lem:dec_power_control} we have:
    \begin{align}
\Ex\lVert{\sum_{j=1}^CW(c,j)\tilde{\bmtheta}^t_j}\rVert^2\leq C(\eta^t)^2\sum_{j=1}^C(W(c,j))^2P_2\mbox{A}
\end{align} 
\label{lem:div_non-noisy_consensus}
\end{lemma}  

\vspace{-5mm}
\section{Experimental Results}
\label{ssec:subhead}
\vspace{-1mm}
In this section, we demonstrate the performance of the proposed CWFL and CWFL-Prox algorithms, focusing on the following: (a) Accuracy across communication rounds, (b) Robustness across different number of clusters and statistical heterogeneity.  We compare the performance of the proposed algorithms with COTAF \cite{COTAF}, FedProx \cite{FedProx} implemented in the wireless framework, which we refer to as COTAF-Prox and single-client training where we assume that all the models train on local dataset and no FL strategy is employed. 
\underline{Datasets and ML models}: We consider the image classification task, where the data is distributed among $K = 25$ clients. We use the popular MNIST and CIFAR10 datasets for bench-marking. In the case of MNIST, the ML model is a convolutional neural network (CNN) consisting of $4$ layers including $2$ convolutional layers, and  batch size $|\omega_k| = 64$. In the case of CIFAR10, the ML model is a CNN consisting of $6$ layers, including $3$ convolutional layers and $|\omega_k| = 32$. Both the architectures use ReLU activation, a learning rate of $\eta = 0.001$. Unless mentioned otherwise, each client has instances pertaining to any $4$ classes in the dataset with $E = 3$. 

\subsection{Accuracy Across Communication Rounds}
In Fig.~\ref{fig:accuracy_cifar} and Fig.~\ref{fig:accuracy_mnist}, we depict the evolution of accuracy across $50$ communication rounds, in the presence of statistical heterogeneity. We observe that the algorithms that use the proximal constraint consistently performs better than its counterparts without the proximal constraint. In  Fig.~\ref{fig:accuracy_cifar}, we observe that CWFL-Prox using $4$ clusters (CWFL-4-Prox) performs similar to COTAF-Prox, and outperforms COTAF. However, we see that CWFL with 3 clusters (CWFL-3-Prox) has slower convergence, but a higher average accuracy as compared to CWFL-4-Prox. This is also evident in Fig.~\ref{fig:BarCluster} (right). On the other hand, in Fig.~\ref{fig:accuracy_mnist} and Fig.~\ref{fig:BarCluster} (left) for  MNIST dataset, we observe that CWFL-3-Prox has a similar performance as compared to COTAF, but COTAF-Prox has the best performance. Here, using $3$ clusters is optimal for CWFL framework. In both the datasets, we see that we outperform the single client training method, where we train each client on the local dataset with no FL strategy being used. In summary, when a strong server is absent, CWFL is a good replacement of COTAF as there is little loss in accuracy. 

\begin{figure}
\centering  \includegraphics[width=0.38\textwidth]{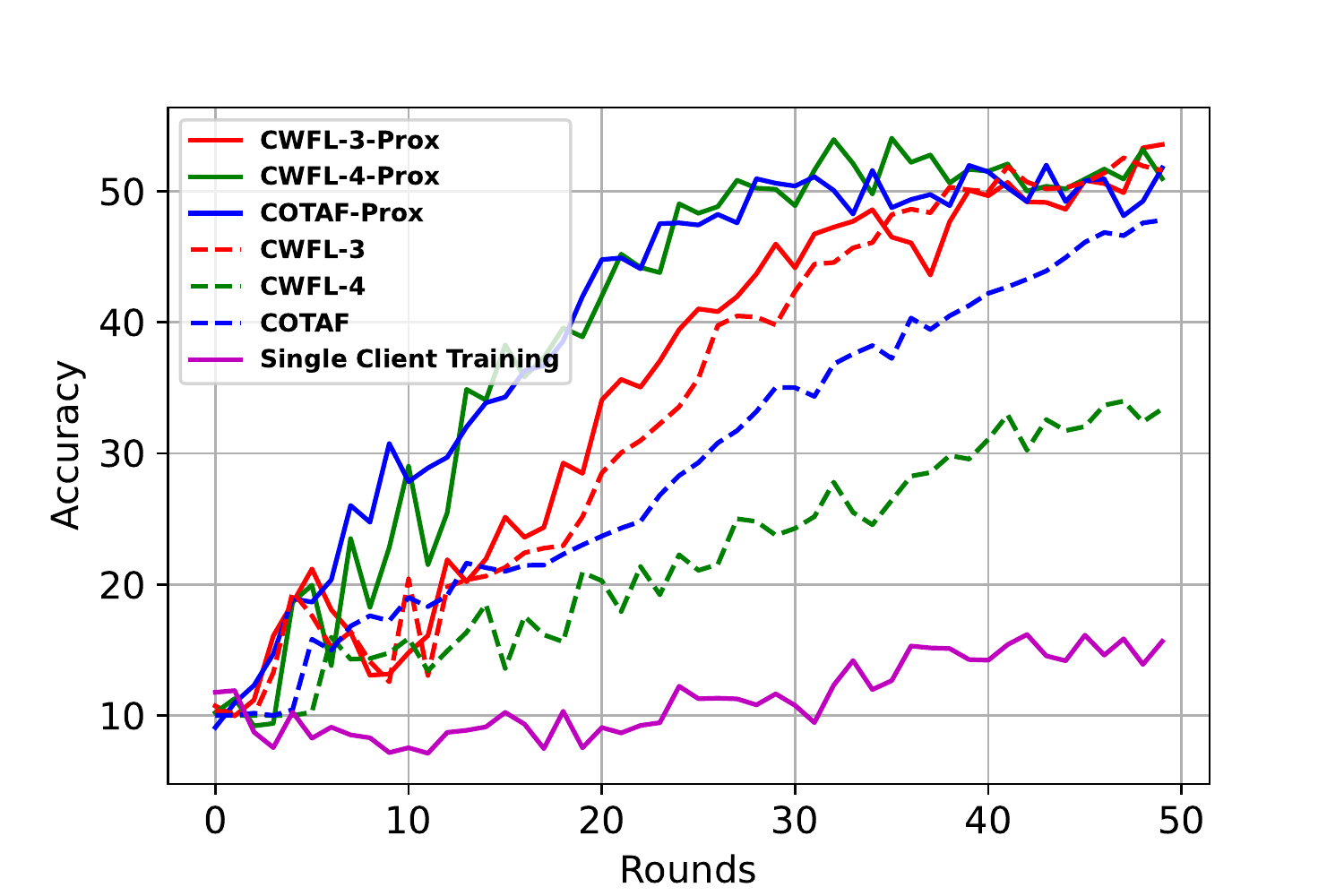}
    \caption{Accuracy convergence  of CWFL and CWFL-prox algorithms on the CIFAR10 dataset.}
    \label{fig:accuracy_cifar} 
    \vspace{-5mm}
\end{figure}

\begin{figure}
\centering  \includegraphics[width=0.38\textwidth]{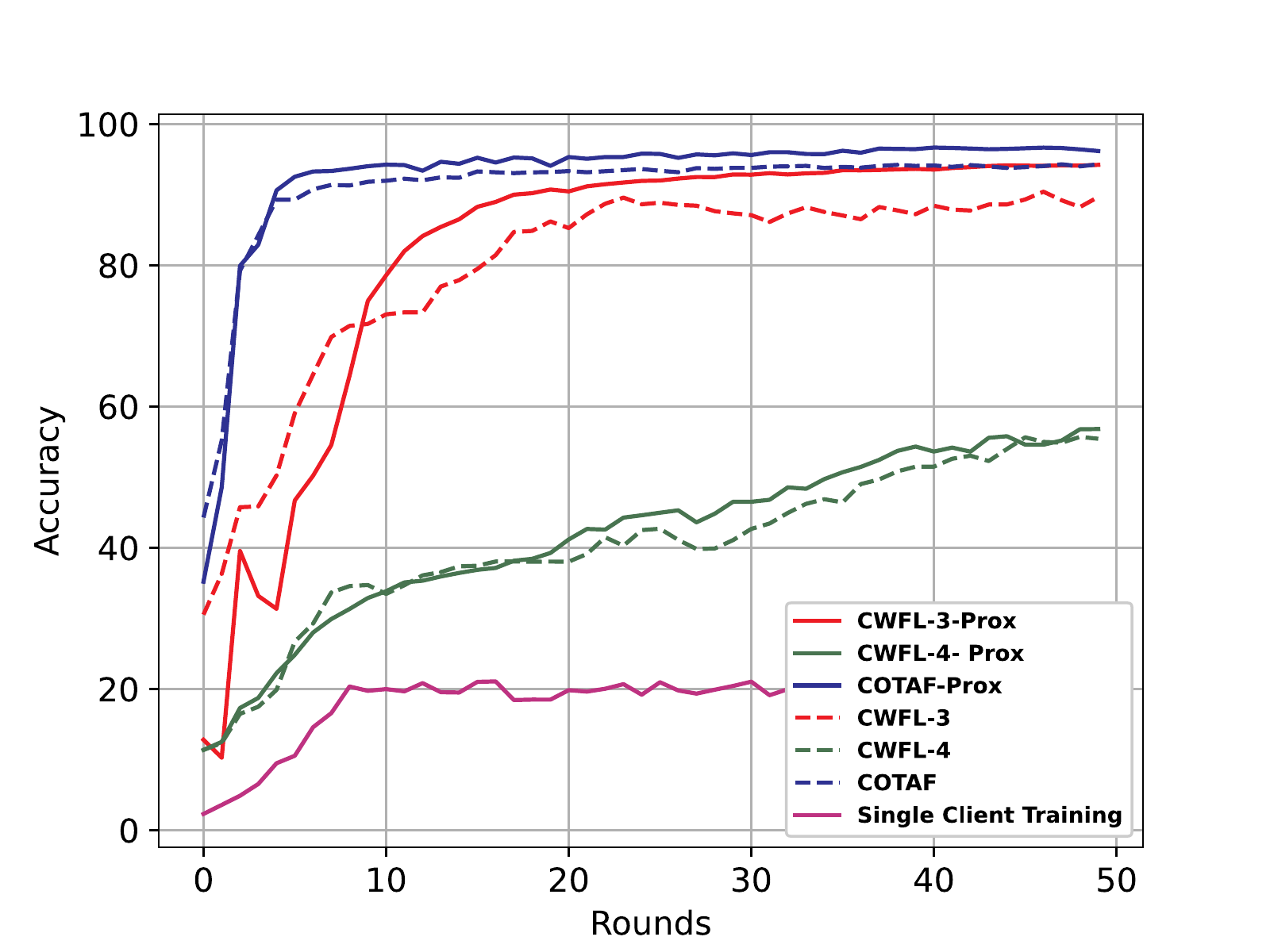}
    \caption{Accuracy convergence  of CWFL and CWFL-prox algorithms on the MNIST dataset.}
    \label{fig:accuracy_mnist}
    \vspace{-5mm}
\end{figure}

\begin{figure}[ht]
\begin{tabular}{ll}
\includegraphics[width=0.24\textwidth]{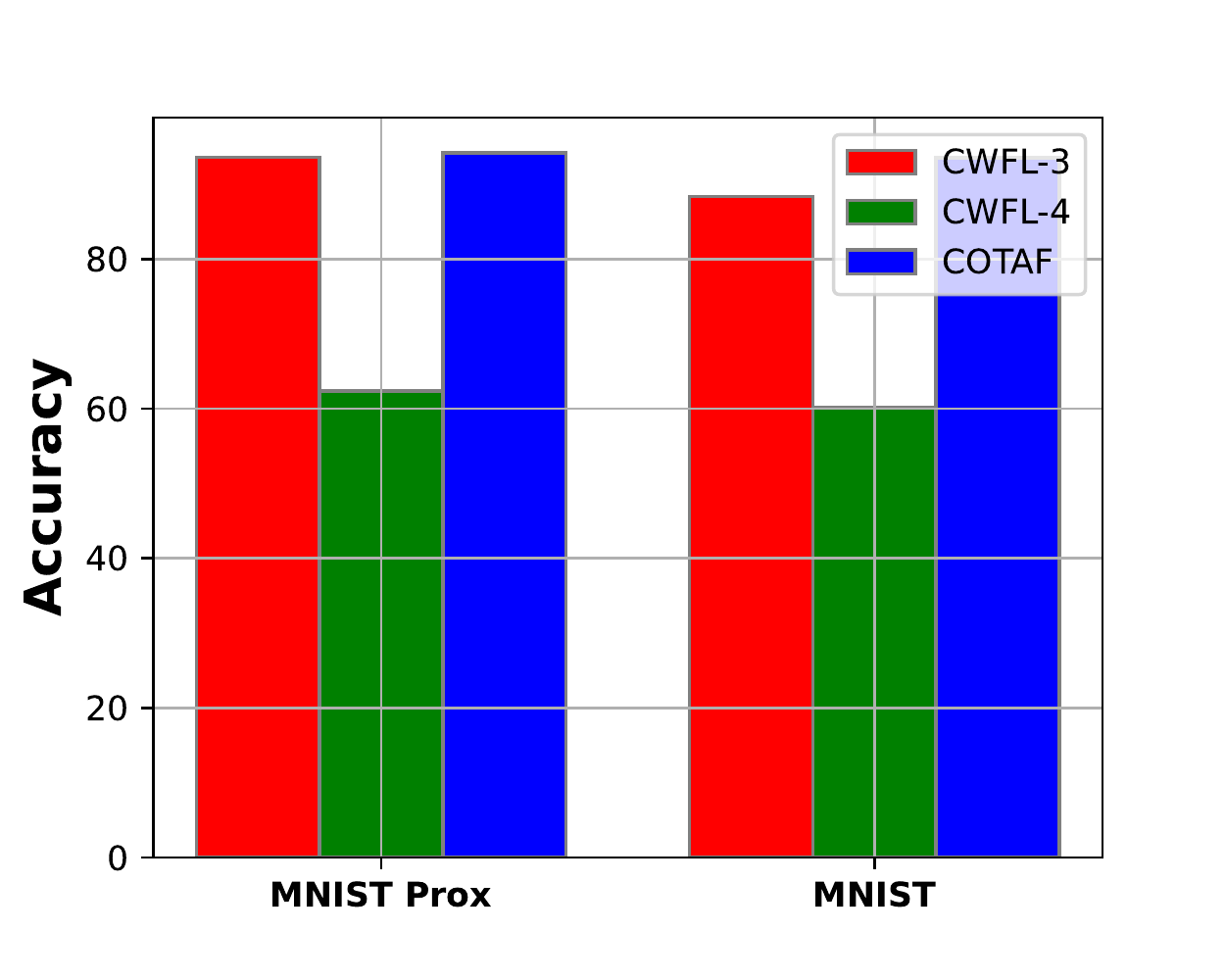}\hspace{-2mm}
&
\includegraphics[width=0.25\textwidth]{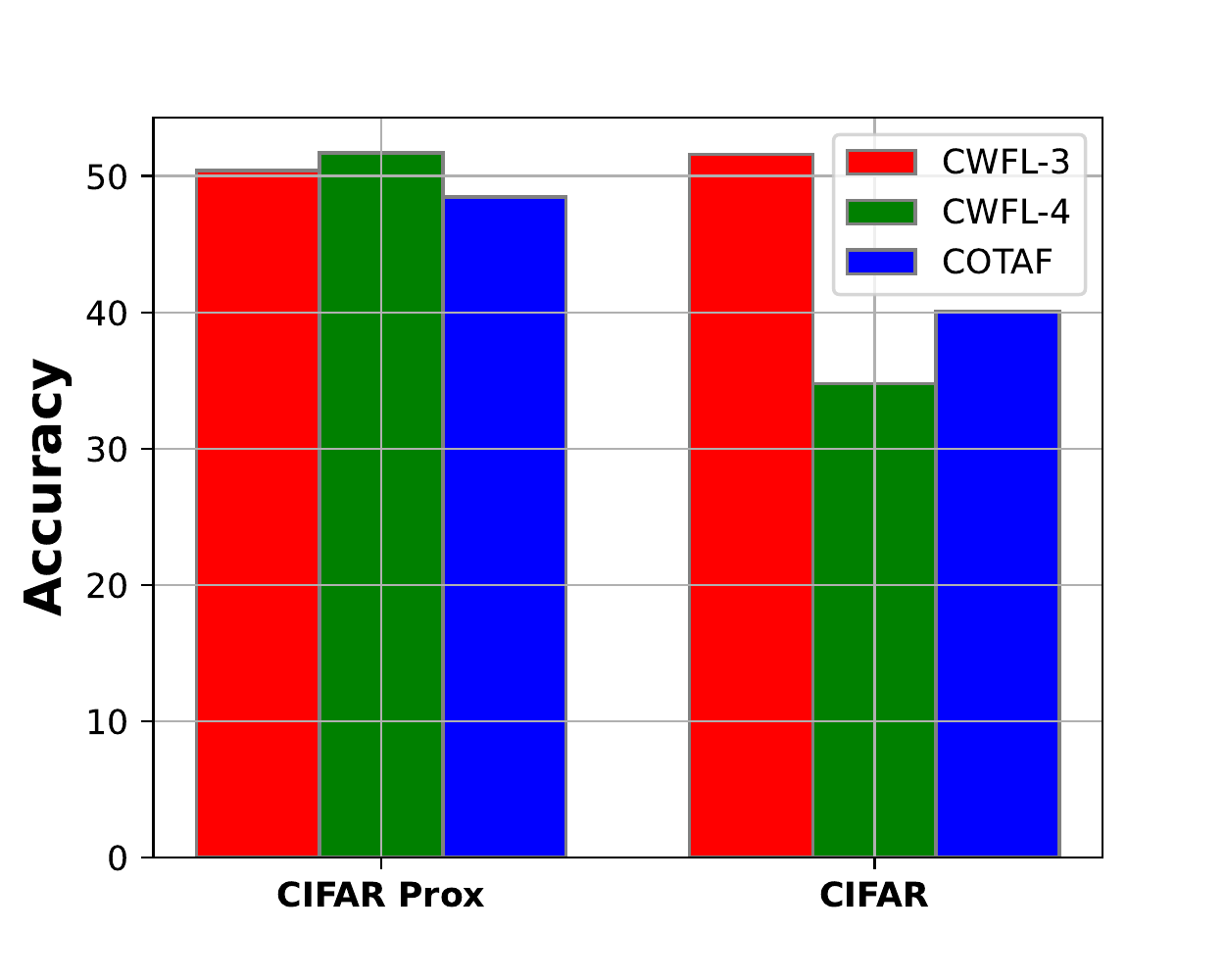}
\end{tabular}
\caption{Accuracy of CWFL and CWFL-Prox for different number of clusters (MNIST and CIFAR10).}
\label{fig:BarCluster}
\vspace{-5mm}
\end{figure}

\begin{figure}[ht]
\begin{tabular}{ll}
\includegraphics[width=0.24\textwidth]{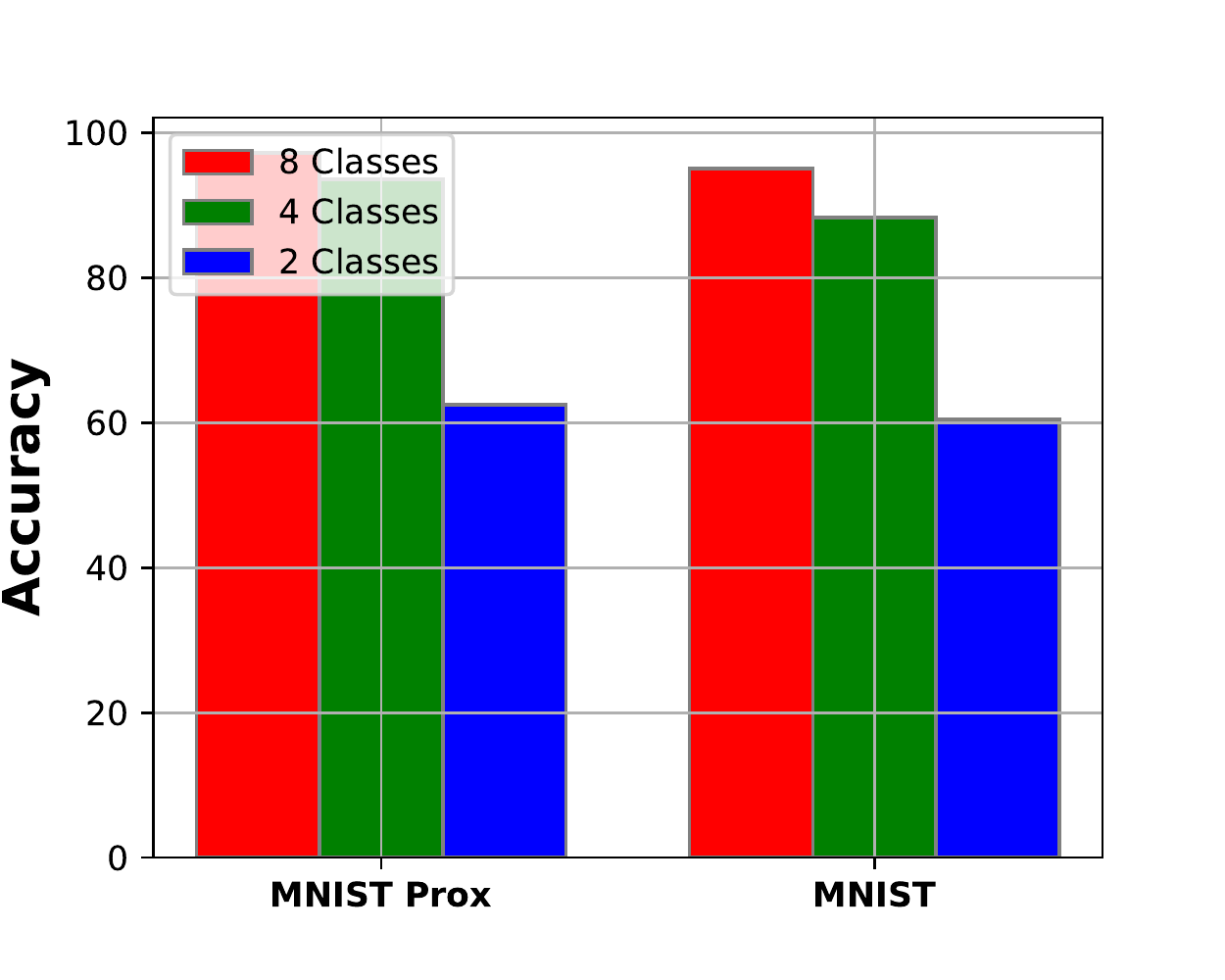}\hspace{-2mm}
&
\includegraphics[width=0.25\textwidth]{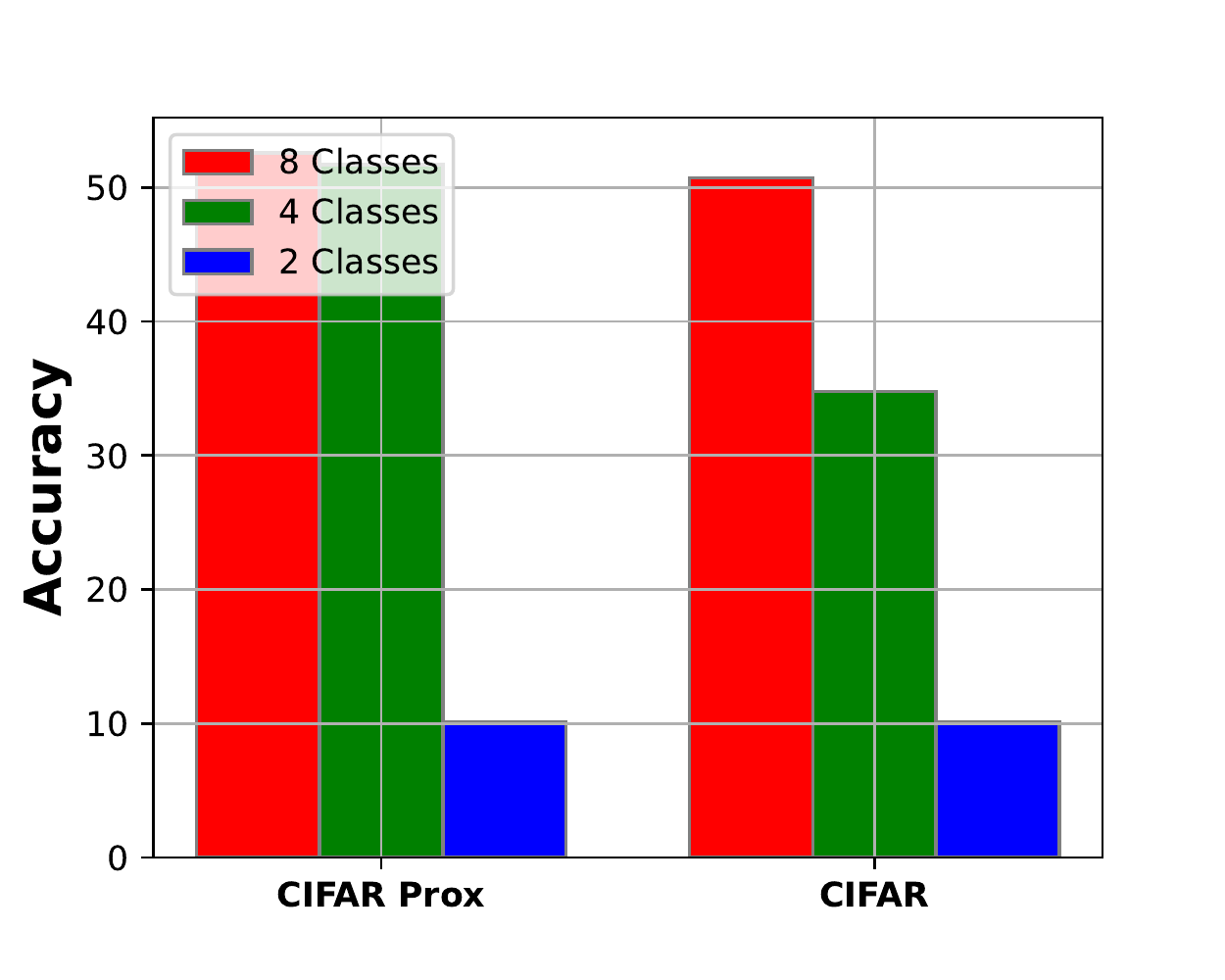}
\end{tabular}
\caption{Accuracy of CWFL and CWFL-Prox for different number of output classes (MNIST and CIFAR10).}
\label{fig:BarClass}
\vspace{-5mm}
\end{figure}

\subsection{Benefits of CWFL}
There are several scenarios where a scheme like CWFL may be preferable as compared to COTAF. For instance, consider the case when server experiences poor SNR conditions as compared to cluster-heads. In order to simulate this scenario, we have set the SNR at the clients to be $1$dB lower than at the server. From Fig.~\ref{fig:Benefit} (left), we see that CWFL performs better in such scenarios by $7$-$8$ dB before finally converging to the same accuracy. This occurs mainly because collectively, at all cluster-heads, the signals experience lower distortion leading to faster consensus. The benefit of using CWFL as compared to the DSGD as employed in \cite{DecentraFLSimeone} is illustrated in Fig.~\ref{fig:Benefit} (right), and the impact is high as the number of clusters increase. By design, $C << K$ and hence, communication complexity of CWFL is lower than DSGD.

\begin{figure}[ht]
\begin{tabular}{ll}
\includegraphics[width=0.25\textwidth]{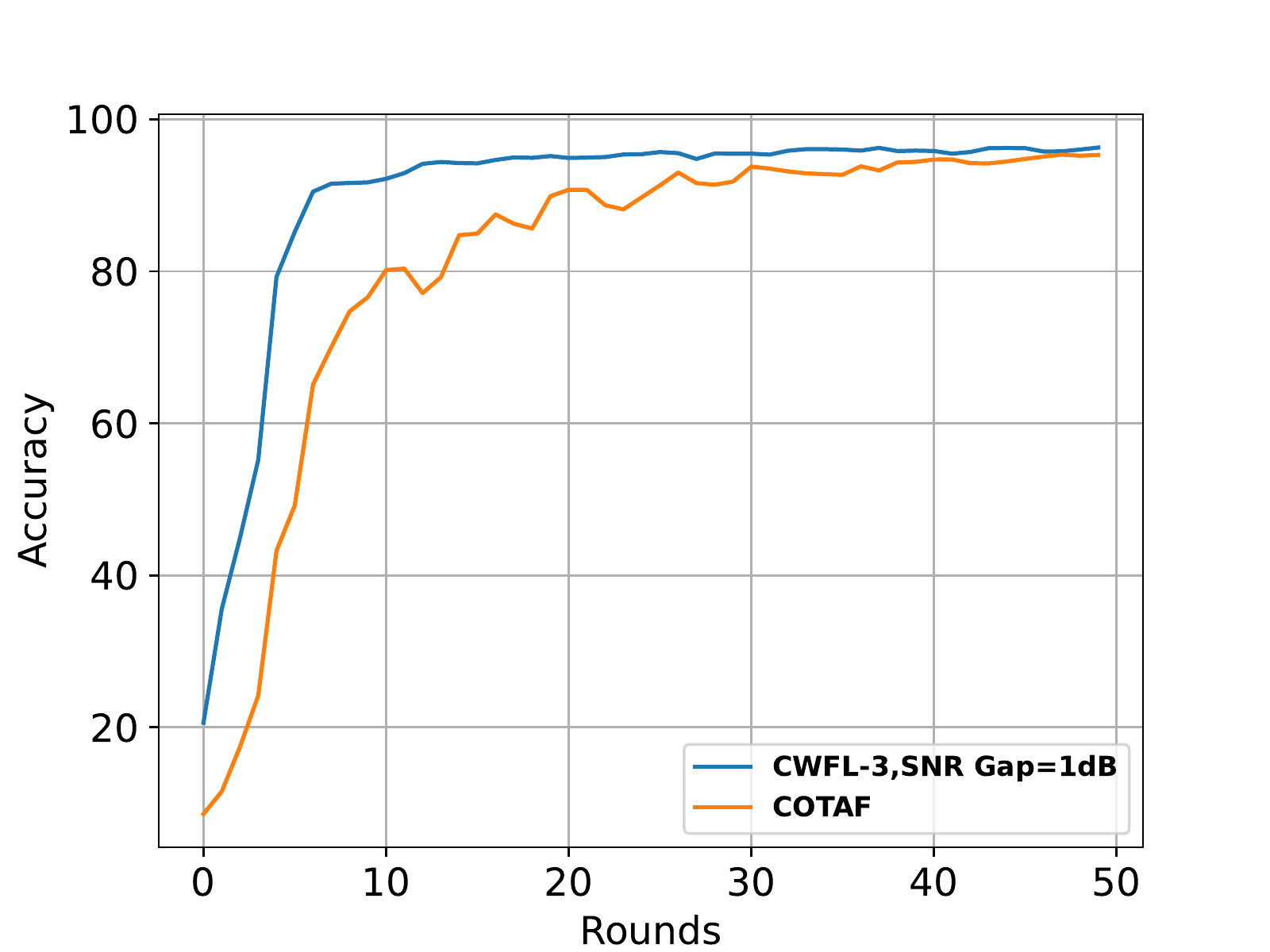}\hspace{-2mm}
&
\includegraphics[width=0.25\textwidth]{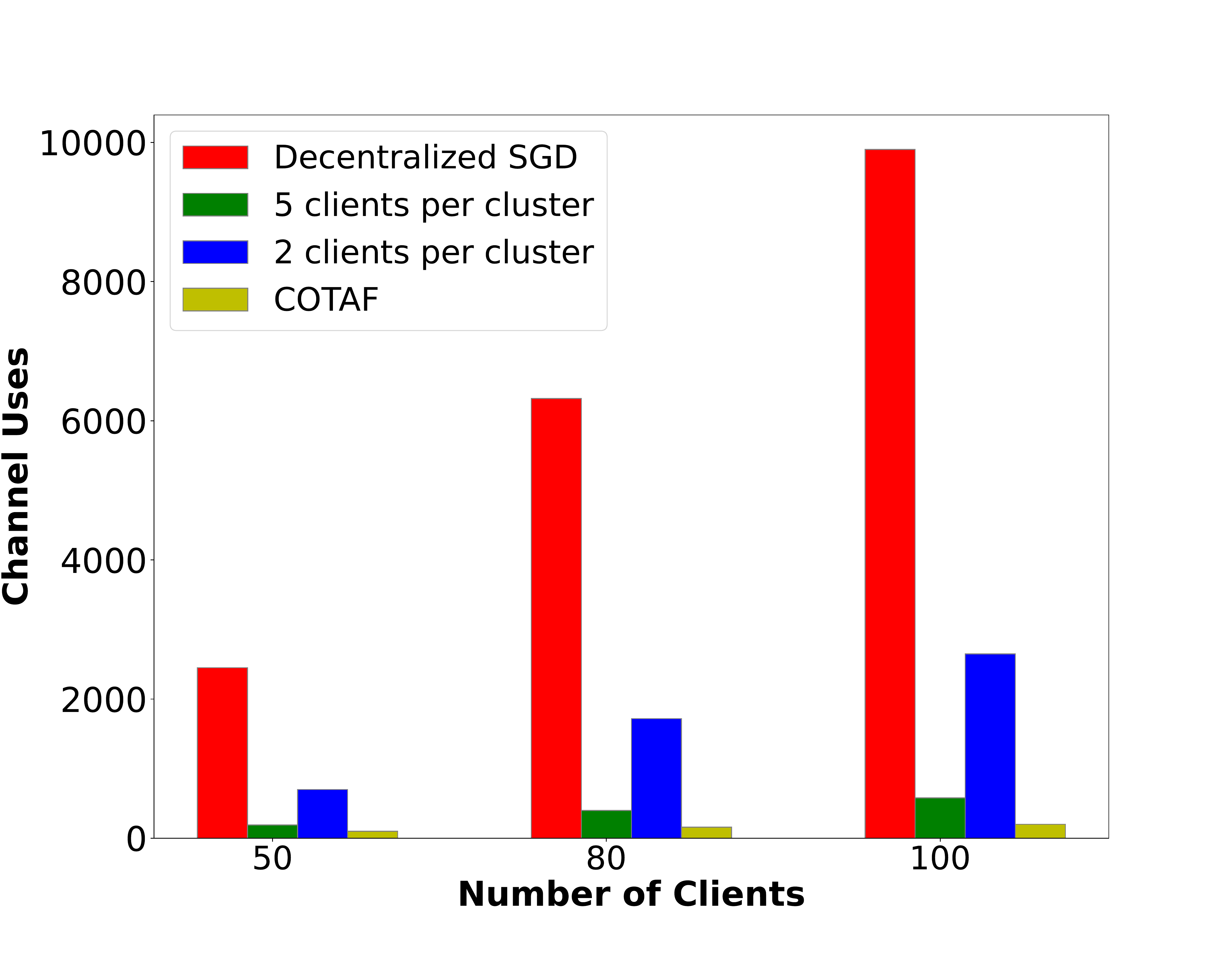}
\end{tabular}
\caption{Benefits of CWFL: Accuracy of CWFL with 1dB SNR gap (left) and communication complexity comparison (right).}
\label{fig:Benefit}
\vspace{-5mm}
\end{figure}

\subsection{Results on Attributes of Federated Learning}
In this section, we study the behavior of the proposed approach with varying number of clusters and varying degree of statistical heterogeneity. Here, the data is partitioned such that each client has access to instances pertaining to any $2$, $4$ or $8$ classes of MNIST and CIFAR10 datasets. The accuracy obtained using CWFL and CWFL-Prox in these cases are depicted in Fig.~\ref{fig:BarClass}. As expected, these algorithms perform well when there is information pertaining to $8$ classes. In challenging datasets such as CIFAR10, the accuracy drops sharply for the pathological $2$ class scenario. For both the datasets, CWFL performs well in the presence of $4$ classes. The performance of CWFL-Prox with $4$ classes is almost as good as CWFL and CWFL-Prox with $8$ classes. 

The accuracy performance of CWFL and CWFL-Prox as compared to COTAF for different number of clusters was investigated, and we observe that optimal number of clusters in this setting is $3$ or $4$. While CWFL-3-Prox outperforms CWFL-4-Prox and COTAF for the CIFAR10 dataset, in the MNIST dataset, the performance of  CWFL-3-Prox is similar to COTAF with respect to average accuracy. In both the cases, it is optimal to choose 3 clusters instead of 4. 
\vspace{-2mm}
\section{Conclusions}
In scenarios where a powerful server is absent, we proposed a semi-decentralised CWFL and CWFL-Prox  frameworks. This framework is based on data-agnostic random clustering of clients, where FL is accomplished using OTA power controlled transmission over wireless uplink channels. We showed that CWFL is a convergent FL scheme ($\mathcal{O}(1/T)$) that requires fewer channel uses as compared to decentralised FL. Using the MNIST and CIFAR10 datasets, we demonstrated that the accuracy performance of CWFL is as good as COTAF, while CWFL-Prox outperforms COTAF. The proximal constraint forces the per-client parameter update to lie in the vicinity of the previous global update, and this mitigates the  combined impact of the noise and statistical heterogeneity.

\bibliographystyle{IEEEbib}
\bibliography{refs}
\end{document}